%% file: POCAII-ArXiV.tex
\documentclass{article}

\usepackage{PRIMEarxiv}

\usepackage[utf8]{inputenc} 
\usepackage[T1]{fontenc}    
\usepackage{hyperref}       
\usepackage{url}            
\usepackage{booktabs}       
\usepackage{amsfonts, amsthm}       
\usepackage{nicefrac}       
\usepackage{microtype}      
\usepackage{lipsum}
\usepackage{fancyhdr}       
\usepackage{graphicx}       
\usepackage{graphicx, subcaption}
\RequirePackage{tgtermes}
\RequirePackage{newtxtext}
\RequirePackage{newtxmath}
\RequirePackage{bm}
\RequirePackage{endnotes}

\usepackage{comment}
\usepackage{multicol,multirow}
\usepackage{algorithm}
\usepackage{algpseudocode}
\usepackage{tikz}
\usepackage{todonotes}
\usepackage{}
\usepackage{natbib}

\usepackage{xspace}
\newcommand{\POCAlg}{\texttt{POCAII}\xspace}

\newtheorem{theorem}{Theorem}[section]
\newtheorem{lemma}[theorem]{Lemma}
\theoremstyle{definition}
\newtheorem{assumption}[theorem]{Assumption}
\newtheorem{proposition}[theorem]{Proposition}

\usepackage{array}
\usepackage{booktabs}
\usepackage{caption}
\usepackage{xparse}
\usepackage{graphicx}

\newcommand{\up}{\vrule height 1.5em depth 0em width 0pt}
\newcommand{\down}{\vrule height 0em depth 0.625em width 0pt}

\title{\POCAlg: Parameter Optimization with Conscious Allocation using Iterative Intelligence
}

\author{
  Joshua Inman \\
  Department of Statistics,\\
  Rice University, \\
  Houston, TX\\
  \texttt{joshua.inman@rice.edu} \\
   \And
  Tanmay Khandait \\
  School of Computing and Augmented Intelligence, \\
  Arizona State University \\
  Tempe, AZ\\
  \texttt{tkhandai@asu.edu} \\
  \AND
    Lalitha Sankar \\
    School of Electrical, Computer and Energy Engineering \\
  Arizona State University \\
  Tempe, AZ\\
  \texttt{lsankar@asu.edu} \\
  \And
  Giulia Pedrielli \\
  School of Computing and Augmented Intelligence, \\
  Arizona State University \\
  Tempe, AZ\\
  \texttt{gpedriel@asu.edu} \\
}

\begin{document}
\maketitle

\begin{abstract}
In this paper we propose for the first time the hyperparameter optimization (HPO) algorithm \POCAlg. \POCAlg differs from the Hyperband and Successive Halving literature by explicitly separating the search and evaluation phases and utilizing principled approaches to exploration and exploitation principles during both phases. Such distinction results in a highly  flexible scheme for managing a hyperparameter optimization budget by focusing on search (i.e., generating competing configurations) towards the start of the HPO process while increasing the evaluation effort as the HPO comes to an end.

\POCAlg was compared to state of the art approaches SMAC, BOHB and DEHB. Our algorithm shows superior performance in low-budget hyperparameter optimization regimes. Since many practitioners do not have exhaustive resources to assign to HPO, it has wide applications to real-world problems. Moreover, the empirical evidence showed how \POCAlg demonstrates higher robustness and lower variance in the results. This is again very important when considering realistic scenarios with extremely expensive models to train.
\end{abstract}


\section{Introduction}\label{sec:Intro}
\input{sections/1_introduction}

\section{Literature Review}\label{sec:lit_rev}
\input{sections/2_literature_review}

\section{Problem Formulation}\label{sec:probForm}
\input{sections/3_problem_formulation}

\section{\POCAlg Method}\label{sec::algoMethodSection} 
\input{sections/4_algorithm_overview}

\section{Asymptotic Behavior Analysis}\label{sec::theory}

\input{sections/5_theory_proof}

\section{Experiments}\label{sec:Results}
\input{sections/6_experiments}

\section{Conclusions and Future Research}\label{sec:Conclusion}
\input{sections/7_conclusion_futureres}

\section*{Acknowledgments}
This work was supported by the National Science Foundation under Grant No. NSF\#2134256. The authors gratefully acknowledge the financial support provided by NSF, which made this research possible. 

\bibliographystyle{plainnat}
\bibliography{ref}  
\end{document}

%% file: sections/1_introduction.tex
Large data-intensive models such as Convolutional Neural Networks (CNNs), Autoencoders, and Deep Belief Networks (DBNs) have become widespread across different domains, including computer vision (CV), natural language processing (NLP), robotics, in a plethora of applications spanning medical care to driving and portfolio management~\citep{zoller2021benchmark}. These models are highly configurable and require careful selection of hyperparameter values (e.g., learning rate, batch size, dropout rate, number of epochs) to achieve optimal performance, measured by metrics such as validation accuracy, F1-score, area under the ROC curve (AUC-ROC), or minimization of validation loss \citep{Feurer2019review, bischl2023hyperparameter}.
This problem is challenging because the validation loss is not known in closed form and therefore it requires training the model to be evaluated given a specific parametrization, which is a knowingly a time demanding process. As a statement to such  complexity, a rich literature has been developed for the algorithmic, sequential generation of candidate hyperparameters, referred to as Hyperparameter optimization (HPO)~\citep{Feurer2019review, bischl2023hyperparameter, raiaan2024hpo}. 

Hyperparameter optimization (HPO) methods aim to automate the search for optimal configurations, with approaches broadly differing in how they balance exploration-exploitation trade-offs and manage computational costs. While the literature encompasses diverse strategies (e.g., evolutionary algorithms, gradient-based methods, meta-learning), two classes of approaches have particularly gained success due to their scalability and empirical success: (1) Bayesian Optimization (BO), and (2) multi-fidelity resource allocation \citep{Feurer2019review}. These approaches address core challenges in HPO, but introduce critical limitations that motivate our work.

Bayesian Optimization (BO) frameworks treat HPO as a surrogate-guided search problem, and is a promising approach since it can treat the loss function as a black box with little to no assumptions on the function properties~\citep{snoek2012practical}. By modeling the loss function using surrogate models, BO iteratively selects hyperparameters. Algorithms such as \cite{snoek2012practical}, \cite{hutter2011sequential}, \cite{snoek2015scalable} demonstrate the flexibility of BO, differing primarily in their choice of surrogate model and acquisition function design. However, BO’s reliance on fast feedback—inexpensive evaluations of the objective function—limits its scalability to computationally intensive tasks (e.g., training deep networks with many epochs), where even a single evaluation is prohibitively costly.

Multi-fidelity resource allocation strategies, such as Successive Halving (SH) \citep{jamieson2016non} and Hyperband (HB) \citep{li2018hyperband}, address this limitation by dynamically allocating resources across configurations. These methods exploit low-fidelity approximations (e.g., training with fewer epochs, subsets of data, or smaller models) to cheaply eliminate poor candidates, reserving full budgets for top performers. However, their reliance on heuristic elimination schedules and assumptions about fidelity-performance correlations can lead to suboptimal allocations, particularly when early performance metrics poorly predict final model quality \citep{raiaan2024hpo}.

Bayesian optimization with hyperband (BOHB)~\citep{falkner2018bohb} bridges the gap between BO and multi-fidelity methods. By integrating Bayesian optimization with Hyperband’s resource allocation, BOHB uses surrogate models to guide the search for promising configurations while dynamically adapting fidelity levels (e.g., epoch budgets). This synergy allows BOHB to outperform standalone BO in low-budget scenarios and match its performance in high-budget settings. Despite its advancements, BOHB inherits unresolved challenges: (1) fidelity schedules (e.g., epoch increments or data subset sizes) are predefined rather than learned, limiting adaptability to problem-specific dynamics; (2) surrogate models may fail to capture complex relationships between fidelity levels and performance, leading to misguided resource allocation; and (3) the interaction between configuration sampling and fidelity selection remains unexplored, leading to redundant evaluations or premature convergence when fidelity correlations are weak.

To address the aforementioned challenges, we introduced the Parameter Optimization with Conscious Allocation (POCA) in our preliminary work in~\cite{inman2023parameter}. POCA improves existing methods by designing a novel optimization budget allocation scheme by shortening certain Hyperbands by reducing the maximum budget a configuration in that Hyperband can receive. Specifically, POCA allocates fewer resources to configurations selected early in the optimization process, i.e., when the number of configurations is small and only noisy loss evaluations are available. This allows more budget to be directed toward promising configurations discovered later in the search, while exploring a large number of cheap, configurations early on. Additionally, POCA emphasizes exploration during the early stages, ensuring broader coverage of the hyperparameter space. Our initial experiments demonstrated that POCA outperforms BOHB during small and medium-budget optimizations.~\citep{inman2023parameter}.

\vspace{3pt}
\begin{center}
\fbox{\parbox{0.9\linewidth}{\textbf{Example. }Consider a case where the total budget available is $B=600$, and the practioner wants no individual configuration to receive a budget larger than $8$ with available budgets of $1, 2, 4, \text{and } 8$. Hyperband or BOHB would run $5$ hyperbands and would sample a total of 100 configurations. On the other hand, POCA would run $6$ hyperbands, saving resources by reducing the budget assigned to early Hyperbands, allowing POCA to test $138$ configurations in total.}}    
\end{center}
\vspace{3pt}

Our analysis identified three critical limitations in the search dynamics of the original POCA framework, which hinder its ability to efficiently identify optimal hyperparameter configurations. First, the algorithm systematically underallocates training budgets to high-performing configurations discovered early in the hyperparameter optimization (HPO) process. Since early-stage discoveries are deprioritized despite their potential, this leads to early discarding of promising candidates, resulting in suboptimal final selections. Second, POCA exhibits excessive exploitation by allocating substantial resources to configurations whose validation performance has plateaued, thereby wasting computational effort on uninformative candidates. Third, the ranking mechanism used to construct the Tree-structured Parzen Estimator (TPE) introduces bias by directly comparing validation metrics of configurations trained to different budgets. Specifically, configurations trained longer are inherently favored over those trained with smaller budgets, even if the latter exhibit superior potential. This bias creates a feedback loop where optimal configurations trained with limited budgets are deprioritized and excluded from further resource allocation. Collectively, these flaws negatively affect the efficiency and fairness of the HPO process, as they imbalance exploration-exploitation trade-offs, misallocate computational resources, and introduce ranking biases. In this work, we present \POCAlg, which significantly addresses these limitations through systematic enhancements to the original POCA framework. Thus, significant improvements to the original version of \POCAlg are available, and are presented in this paper.

This work contribution is twofold:

\begin{itemize}
    \item[\textbf{C1}] \textbf{A Novel Structure for Configuration Search and Evaluation:} We introduce a new version of \POCAlg that departs from the original, rigid Hyperband structure. Due to that structure, promising configurations sampled early on in POCA's HPO process would not be evaluated further, preventing the algorithm from identifying the strongest configuration found. Instead, the method alternates between two phases: configuration search and configuration evaluation. During the search phase, hyperparameter configurations are sampled from the search space and trained to a small budget (i.e. epochs). During the evaluation phase, promising configurations previously sampled in a search phase are selected for further evaluation, enabling configurations found early on to be trained on a larger budget. However, this approach retains the original version's philosophy of prioritizing search early in the optimization process and transitioning to evaluation later by gradually increasing the budget of evaluation phases.
    
    \item[\textbf{C2}] \textbf{Resource Allocation Based on Configuration Performance:} In the evaluation phase, \POCAlg re-selects configurations based on both their potential to be the best configuration and whether their validation loss is still improving as additional budget is assigned. This prevents \POCAlg from wasting resources on configurations whose validation metric has plateaued, unlike the original POCA algorithm. \POCAlg employs simple time-series models for this purpose, allowing better comparisons between configurations of different budgets by considering the trend of a configuration as it receives more resources.
\end{itemize}

%% file: sections/2_literature_review.tex
Various algorithms have been proposed for hyperparameter optimization (HPO), attracting significant research and commercial interest~\citep{baratchi2024automated}. The algorithms addressing the HPO problem can broadly be categorized into model-free approaches, Bayesian optimization (BO) approaches, multifidelity optimization approaches, and gradient descent-based approaches. The remainder of this section provides a brief overview of the existing work in this area.

\subsection{Model-free approaches}\label{sec:modelFree}

Model-free hyperparameter optimization (HPO) methods generate candidate configurations without relying on a surrogate model of the validation loss. Traditional approaches include grid search, random search, evolutionary algorithms, and gradient-based optimization, etc. While there is a rich literature in the context of black box optimization for these techniques, we focus on implementations that have been proposed specifically within the HPO context.  Grid search, the simplest method, exhaustively evaluates hyperparameter combinations on a predefined grid with evenly spaced values. However, its computational inefficiency in high-dimensional spaces—a consequence of the curse of dimensionality—limits its practicality \citep{bergstra2012random}. To address this, adaptive grid search methods refine resolution dynamically, such as the coarse-to-fine strategy proposed by \citet{larochelle2007randommultires}, which iteratively zooms into promising regions. Despite such adaptations, rigid grid structures remain less popular than flexible sampling strategies in modern HPO.

An alternative to defining a rigid grid-like neighborhood structure is to design sampling distributions in the continuous (or mixed) hyperparameter space. Within this family, random search samples hyperparameters from user-defined distributions (e.g., uniform or log-uniform) and often outperforms grid search in high-dimensional settings, particularly when only a subset of hyperparameters significantly impacts performance \citep{bergstra2011algorithms}. Extensions leverage quasi-random sequences like Sobol or Halton to achieve low-discrepancy sampling, improving coverage of the search space compared to uniform random sampling \citep{bergstra2012random}. Frameworks such as Hyperopt \citep{bergstra2013hyperopt} employ these strategies, though they differ from model-based approaches like SMAC \citep{lindauer2022smac3}, which combine sampling with surrogate models. Gradient-based methods, while widely used for training neural networks, face challenges in HPO due to the computational expense of computing hypergradients (derivatives of the validation loss with respect to hyperparameters). Recent advances, such as implicit differentiation through training trajectories \citep{lorraine2020optimizing} and nested automatic differentiation techniques \citep{franceschi2017forward, maclaurin2015gradient}, enable scalable gradient-based HPO for specific model classes, including deep neural networks.

Evolutionary Strategies (ES) have been adapted to HPO tasks by directly targeting challenges like mixed search spaces, computational budgets, and noisy evaluations. One of the most recent works HPO-specific ES is Regularized Evolution \citep{real2019regularized}, which introduces aging mechanisms to discard outdated configurations, enabling robust optimization in high-dimensional spaces (e.g., neural architecture search). This approach outperformed random search and Bayesian optimization on benchmarks like NASBench-101, highlighting ES’s suitability for structured HPO tasks. For computationally expensive evaluations, \cite{awad2021dehb} combines evolutionary operators with successive halving to dynamically allocate resources, reducing wall-clock time while maintaining competitive performance. While ES avoids the gradient-computation bottlenecks of model-based methods, its population-based nature still incurs overhead compared to simpler strategies like random search, particularly when training costs are prohibitive \citep{raiaan2024hpo}.

\subsection{Bayesian Optimization approaches}\label{sec:BOmethod}

Bayesian optimization (BO) has emerged as a dominant paradigm for hyperparameter optimization (HPO), particularly in computationally expensive settings \citep{Feurer2019review, raiaan2024hpo}. Unlike model-free methods, BO employs surrogate models, such as Gaussian processes (GPs) or random forests, to approximate the validation loss landscape and strategically select hyperparameters that balance exploration and exploitation. Work by \cite{snoek2012practical} established BO’s superiority over grid and random search for HPO, demonstrating its ability to optimize neural network hyperparameters with minimal evaluations. Their GP-based framework, implemented in the Spearmint toolbox, became a cornerstone of automated HPO, particularly for black-box functions where gradient information is unavailable or costly to compute.

Scalability limitations of GPs in high-dimensional spaces due to cubic computational complexity gave rise to methods tailored to HPO. For instance, \cite{hutter2011sequential} introduced SMAC, replacing GPs with random forests to handle discrete and conditional hyperparameters, while \cite{bergstra2013hyperopt} proposed Tree-structured Parzen Estimators (TPE) in Hyperopt for asynchronous and scalable BO. These broadened BO’s applicability to complex HPO workflows, including pipeline configuration and resource-constrained optimization. To further reduce wall-clock time, multi-fidelity techniques such as BOHB \citep{falkner2018bohb} integrated BO with Hyperband, dynamically allocating resources across configurations based on low-fidelity approximations (e.g., training on subsets of data). Parallelization strategies, including batch acquisition via local penalization \citep{gonzalez2016batch}, enabled distributed HPO, making BO feasible for industrial-scale tasks.

Recent advances extend BO to structured and high-dimensional HPO problems, such as neural architecture search (NAS) and conditional parameter spaces. Trust-region BO (TurBO) \citep{eriksson2019scalable} partitions high-dimensional spaces into local regions for efficient optimization, while methods like COMBO \citep{oh2018bock} employ graph-based kernels to handle categorical and combinatorial hyperparameters. These extensions lie at the base of modern AutoML frameworks, including Auto-sklearn \citep{feurer2019autosklearn} and Google Vizier \citep{googlevizier}, which automate HPO and pipeline configuration at scale. Industrial adoption highlights BO’s practicality: Vizier, for example, supports thousands of concurrent trials with multi-objective and constrained optimization \citep{akiba2019optuna}.

Despite its successes, BO faces challenges in HPO contexts, including sensitivity to surrogate model choice, degraded performance in high dimensions ($>100$ hyperparameters), and cold-start reliance on initial random evaluations. Hybrid approaches, such as coupling BO with meta-learning \citep{perrone2018scalable} or neural surrogates \citep{springenberg2016bayesian}, aim to mitigate these limitations, though scalability and generalizability are the major bottlenecks.

\subsection{Multifidelity approaches}\label{sec:MFidOpt}

Multi-fidelity hyperparameter optimization (HPO) methods address computational bottlenecks by strategically allocating resources to evaluate configurations at varying fidelity levels, such as training epochs, dataset subsets, or model complexity. These approaches prioritize promising configurations early, discarding underperformers before they incur high computational costs. By balancing exploration (testing diverse configurations) and exploitation (refining top candidates), multi-fidelity HPO achieves scalability in high-dimensional search spaces~\citep{klein2017fast}.

Foundational algorithms like Successive Halving (SH)~\citep{jamieson2016non} and Hyperband (HB~\citep{li2018hyperband} established the paradigm of iterative resource allocation. SH eliminates the worst-performing half of configurations at increasing budgets, while Hyperband automates this process by running multiple SH ``brackets'' with dynamically adjusted budgets. For example, Hyperband evaluates hundreds of configurations at low fidelity (e.g., 1 epoch) and progressively allocates full training budgets to survivors, achieving robustness without manual tuning. This framework has become a benchmark for resource-constrained HPO, particularly in low-budget scenarios~\citep{li2018hyperband}.

Model-based extensions enhance these foundations. BOHB~\citep{falkner2018bohb} integrates Bayesian optimization (BO) with Hyperband, replacing random search with Tree-structured Parzen Estimators (TPE) to propose high-potential configurations. This hybrid retains Hyperband’s efficiency while leveraging probabilistic models to focus on promising regions of the search space. Similarly, FABOLAS~\citep{klein2017fast} uses dataset size as a fidelity parameter, modeling validation loss as a function of both hyperparameters and data subsets to extrapolate performance from small samples. These innovations demonstrate how surrogate models can refine multi-fidelity resource allocation.

Scalability and parallelism are addressed by methods like ASHA (Asynchronous Successive Halving)~\citep{li2020system} and DEHB (Differential Evolution Hyperband)~\citep{awad2021dehb}. ASHA parallelizes SH by asynchronously discarding underperforming trials, enabling distributed HPO across clusters. DEHB combines evolutionary strategies with Hyperband, using mutation and crossover to explore mixed categorical-continuous search spaces common in neural architecture search (NAS). These methods highlight the adaptability of multi-fidelity frameworks to diverse computational environments and problem structures.

Recent advances focus on automating fidelity adaptation and reducing wasted computation. HyperJump~\citep{mendes2023hyperjump} trains neural networks to predict early-stopping points for configurations, terminating trials unlikely to improve, while MFES-HB~\citep{li2021mfes} meta-learns fidelity schedules across tasks to minimize cold-start overhead. Such innovations address challenges like non-stationary loss landscapes and task-specific tuning, pushing multi-fidelity HPO toward greater autonomy.

%% file: sections/3_problem_formulation.tex
We consider the problem of finding a set of hyperparameter configurations $\boldsymbol{\lambda}^*$ for a given machine learning model $\mathcal{A}$ that minimizes a predefined loss function $\mathcal{L}$, while satisfying computational constraints. Specifically, given a dataset $D = \left\lbrace \left(x_i,y_i\right)\right\rbrace^{N}_{i=1}$ with feature vector $x_i$ and labels $y_i$, a machine learning model $\mathcal{A}$ with a set of hyperparameters $\Lambda$ ($\Lambda$ can be generally formed by continuous, categorical and discrete variables), a validation set $D_{\mbox{\tiny{val}}}$ used to estimate the model performance, and a loss function $\mathcal{L}:\Lambda \times D_{\mbox{\tiny{val}}}\rightarrow \mathbb{R}$, we want to solve the following:

\begin{align}    \boldsymbol{\lambda}^*\in\arg\inf_{\lambda\in\Lambda}\mathcal{L}\left(\mathcal{A}\left(\lambda\right),D_{\mbox{\tiny{val}}}\right)\label{eqn::autoMLprob}.
\end{align}
In equation~\eqref{eqn::autoMLprob}, $\mathcal{A}\left(\boldsymbol{\lambda}\right) = \sup_{\beta \in \mathbb{R}^+}\mathcal{A}\left(\boldsymbol{\lambda}, \beta\right)$ is the model trained under hyperparameters $\boldsymbol{\lambda}$ (a hyperparameter solution, referred to as \textit{configuration}), under a large budget regime (i.e., a budget at which the model achieves its minimum validation loss). In particular, $\beta$ represents any measurable and non-decreasing resource associated with training. The budget $\beta$ can be expressed in terms of epochs, training time, or memory use. 
The validation loss is $\mathcal{L}\left(\mathcal{A}\left(\lambda\right),D_{\mbox{\tiny{val}}}\right)$. Generally, we cannot evaluate the model true loss since validation resources are typically scarce. In fact, letting $\beta^{\boldsymbol{\lambda}}$ be the \textit{total} budget assigned to configuration $\boldsymbol{\lambda}$, we desire the optimization in equation~\eqref{eqn::autoMLprob} to satisfy a further budget constraint $\sum_{\boldsymbol{\lambda} \in \Lambda} \beta^{\boldsymbol{\lambda}} \le B$, where $B$ is the total budget available. 

In this work's experimentation (Section \ref{sec:Results}), we use epochs for the budget and (negative) validation accuracy for the validation loss.

%% file: sections/4_algorithm_overview.tex
The \POCAlg algorithm alternates between two phases, the \textbf{search}, and the \textbf{evaluation} phase until the user-defined budget $B$ (i.e. total number of epochs) is exhausted. Figure~\ref{fig:flowchartPOCAII} shows the \POCAlg flowchart. At iteration $k$, the \textbf{search} phase (Section \ref{sec:search}) takes as input the space of feasible configurations $\Lambda \setminus \Sigma_{k-1}$ (here, $\Sigma_{k-1}$ is the set of configurations sampled up to iteration $k$), and it generates $n_{\mbox{\tiny{S}}}$ candidate configurations $\left\lbrace\boldsymbol{\lambda}_{i}\right\rbrace^{(k-1)n_{\mbox{\tiny{S}}} + n_{\mbox{\tiny{S}}}}_{i=(k-1)n_{\mbox{\tiny{S}}} + 1}$, where $n_{\mbox{\tiny{S}}}$ is a user input and $(k-1)n_{\mbox{\tiny{S}}}$ is the number of configurations sampled in previous iterations. Hence, we index each configuration we sample from $i=1,\ldots,kn_{\mbox{\tiny{S}}}$. These configurations are generated by means of a sampling scheme, designed to balance exploration and exploitation, which utilizes a Tree Parzen Estimator (TPE) (section~\ref{sec:backGroundPOCA}) and uniform random sampling. All the generated configurations are then evaluated by training the associated learning model $\mathcal{A}\left(\lambda_{i}\right)$ using a budget of $\delta$, where $\delta$, a user-defined parameter, is the minimum budget we can use to train a configuration. $\delta$ may be a function of the iteration or the budget expended so that the minimum budget configurations receive grows throughout the optimization process. As a result of the search phase, the set of candidate configurations $\Sigma_{k} \subseteq \Lambda$ is updated, i.e., $\Sigma_{k} \leftarrow \Sigma_{k-1} \cup \{{\lambda}_{i}\}_{i=(k-1)n_{\mbox{\tiny{S}}}+1}^{(k-1)n_{\mbox{\tiny{S}}} +n_{\mbox{\tiny{S}}}}$. Each configuration is saved with the associated loss value after iteration $k$, which we call $\mathcal{L}^{ik}$ for configuration $\lambda_i$. In addition, each configuration's validation loss is measured several times prior to budget $\delta$ so that a sequence of validation losses $\{\mathcal{L}^{i}_j\}_{j=\frac{\delta}{m}}^{\delta}$, where $\frac{\delta}{m}$ is the budget interval at which the validation loss is measured. (Hence, we use $\mathcal{L}^{ik}$ to be the validation loss at the largest budget $\lambda_i$ received through iteration $k$, and $\mathcal{L}^i_j$ to be the validation loss of $\lambda_i$ after receiving a budget of precisely $j$). Finally, we also save the learning models $A(\lambda_i)_k$ so that training may be resumed from the budget $\lambda_i$ was trained to at the end of iteration $k$.

The \textbf{evaluation} phase (Section \ref{sec:eval}) uses information from configurations previously generated in the search phase to identify and evaluate promising configurations by allocating additional resources (i.e., training epochs) to a subset of configurations. In this regard, the evaluation focuses on the loss sequences $\{\mathcal{L}^{ik}_{j}\}_{j=\frac{\delta}{m}}^{\beta^{ik}}$, where $\beta^{ik}$ is the total budget $\lambda_i$ has received at this point in iteration $k$. As previously mentioned each sequence is initialized at $\beta^{ik}=\delta$. The phase is initialized by fitting a time series model(in this version of \POCAlg, we use an ARIMA) to the loss sequence available for each configuration in the set $\Sigma_k$. Then, for each $\lambda \in \Sigma_k$, a probability of being selected for further training $p_{\lambda_{i}}$ is assigned based on the configuration's expected improvement over the current best configuration $\lambda^*_k$, such that $\sum_{\lambda \in \Sigma_k} p_\lambda = 1$. Then, by sampling $k$ times from this probability distribution, $k$ configurations are selected and trained with an additional budget of $\delta$. After training commences, the TPE used in the search phase is re-fit to include the new data, and POCA advances to the next Search Phase.

\begin{figure}
\centering
\includegraphics[width=0.8\textwidth]
{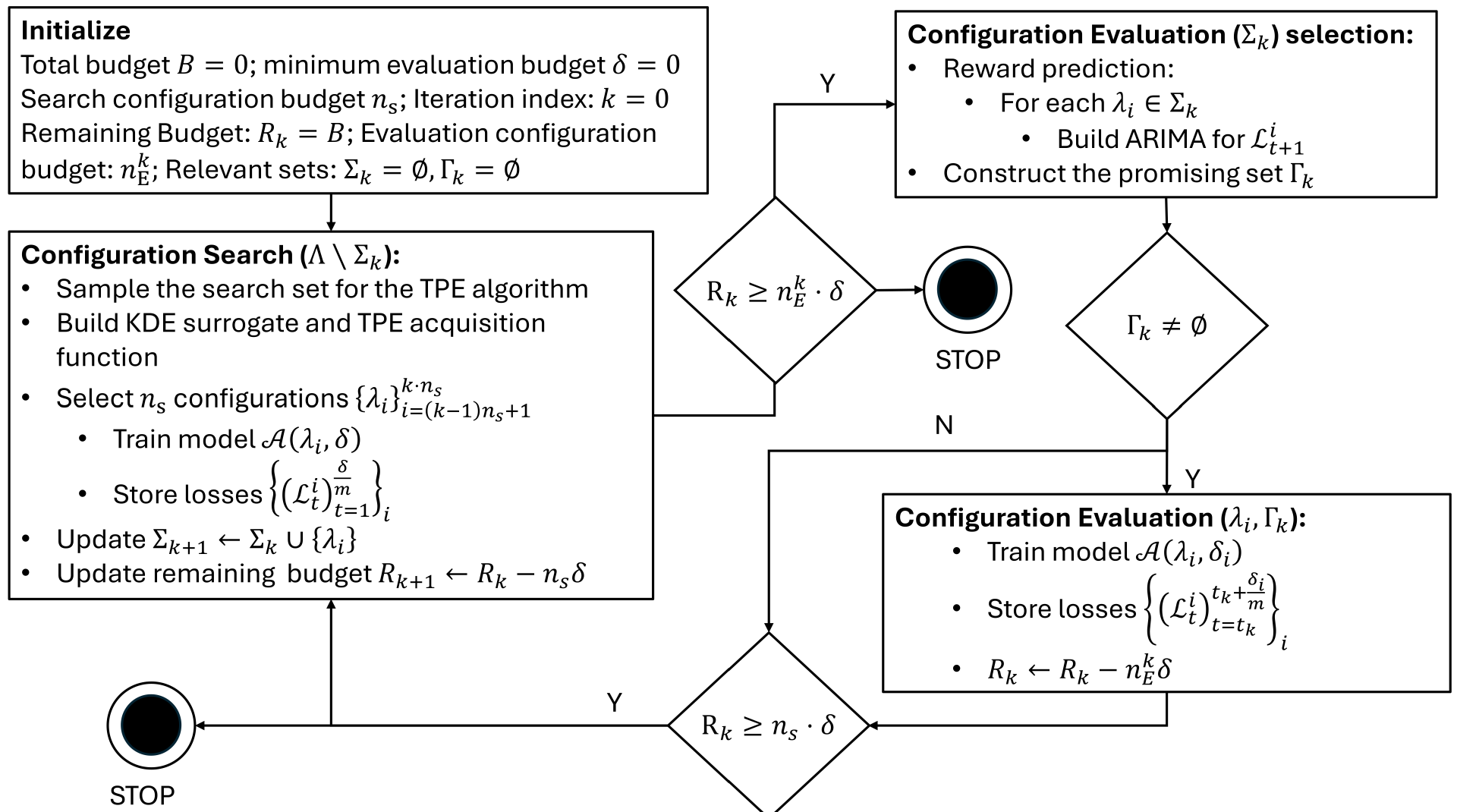}
\caption{Algorithm flowchart of \POCAlg. When the algorithm stops it first uses the remainder epochs allocating them proportionally to the response improvement predicted for each configuration. Finally the configuration with the best associated loss value is returned as incumbent.}
\label{fig:flowchartPOCAII}
\end{figure}

Section \ref{sec:backGroundPOCA} explains the relevant background,i.e., non-original methods that \POCAlg phases require. The Search Phase is presented in Section~\ref{sec:search}, while the Evaluation Phase in Section \ref{sec:eval}. Section~\ref{sec:implement} provides implementation considerations to maximize the experimental performance of the proposed approach.

\subsection{Background}\label{sec:backGroundPOCA}
This section provides the foundational background for the \POCAlg algorithm. We first discuss the Tree-structured Parzen Estimators (TPE), which POCA uses for sampling promising configurations. Then, the section concludes with a discussion on Hyperband, primarily used to demonstrate the asymptotic guarantees we present in Section~\ref{sec::theory}.

\paragraph{Tree Parzen Estimation. }  The Tree Parzen Estimator (TPE) is a type of Bayesian optimizer \citep{bergstra2011algorithms}, which works by iteratively sampling a subset of locations from the feasible region and evaluating an acquisition function for each candidate within the subset. 

Specifically, a TPE iterate starts by sampling uniformly at random from the set of feasible configurations $\Lambda \setminus \Sigma_{k}$, a subset of $n_{\mbox{\tiny{TPE}}}$ points. We refer to this subset as $S^{\mbox{\tiny{TPE}}}_k$. TPE ranks the configurations $\lambda \in S^{\mbox{\tiny{TPE}}}_k$ based on their probability of improvement $P\left(\mathcal{L}\left(\lambda\right)\le \mathcal{L}^*_k\right)$, where $\mathcal{L}^*_k = \min_i \mathcal{L}^{ik}$, and $\lambda_k^{*}$. Rather than estimate $\mathcal{L}(\lambda)$ directly, the TPE uses kernel density estimators to construct a separate acquisition function equivalent in rank to probability of improvement. 

To use this alternative acquisition function, a TPE first splits the data $\Sigma_{k-1}$  into ``good'' and ``bad'' configurations $\Sigma^{\ell}_{k}, \Sigma^{g}_{k}$, respectively. (Note, $|\Sigma_{k-1}|$ must be sufficiently large; in particular, $|\Sigma_{k-1}| > d$, where $d$ is the number of hyperparameters (the dimension of $\Lambda$). Specifically, $\Sigma^{\ell}_{k}=\{\lambda^{i}:\lambda^{i}\in\Sigma_k,\mathcal{L}^{i}_k \le q^{\mathcal{L}}_{\gamma} \}$, where $q^{\mathcal{L}}_{\gamma}$ is the $\gamma$-percentile loss within $\Sigma_k$. Similarly, $\Sigma^{g}_{k}=\{\lambda^{i}:\lambda^{i}\in\Sigma_k,\mathcal{L}^{i}_k > q^{\mathcal{L}}_{\gamma} \}$. 

A kernel density estimator (KDE) builds an estimate for the posterior $p(\lambda | \Sigma^{\ell}_k), \lambda\in \Lambda\setminus \Sigma_k$, using an uninformative prior $p_0(\lambda)$ (which is commonly a uniform prior), which encourages exploration. Here, rather than creating a posterior $p(\lambda | \mathcal{L}, \Sigma_k^\ell)$ for each $\mathcal{L} \le q_{\gamma}^{\mathcal{L}}$, the TPE makes the critical assumption that $p(\lambda | \mathcal{L}, \Sigma_k^\ell) = p(\lambda | \Sigma_k^\ell)$, so only one posterior distribution is necessary. Letting $n^{\ell}_k = |\Sigma^{\ell}_k|$, following the KDE scheme, a natural estimate for the posterior probability of a location in $S^{\mbox{\tiny{TPE}}}_k$ to be a poor configuration is:
\begin{equation}
    p(\lambda | \Sigma^{\ell}_k; \upsilon) = \frac{1}{n^{\ell}_k+1}p_0(\lambda) + \sum_{\lambda_i \in \Sigma^{\ell}_k} \frac{1}{n^{\ell}_k+1} \kappa(\lambda_i, \lambda ; \upsilon),\label{eqn::badLikely}
\end{equation}

where $\kappa$ is kernel function of choice with bandwidth hyperparameter $\upsilon$, and $\lambda_i \in \Sigma^{\ell}_k$ is a sampled configuration within the set $\Sigma^{\ell}_k$. 

The same posterior density can be derived for $\lambda\in\Lambda$ considering losses that satisfy $\mathcal{L} > q^{\mathcal{L}}_{\gamma}$. In this case, we obtain the following posterior:

\begin{equation}
    p(\lambda | \Sigma^{g}_k; \upsilon) = \frac{1}{n^{g}_k+1}p_0(\lambda) + \sum_{\lambda_i \in \Sigma^{g}_k} \frac{1}{n^{g}_k+1} \kappa(\lambda_i, \lambda ; \upsilon).\label{eqn::goodLikely}
\end{equation}

Then, for a sampled sampling set $S^{\mbox{\tiny{TPE}}}_k$, the TPE acquisition function is calculated as the likelihood ratio:

\begin{equation}
    r(\lambda | \Sigma_k;\upsilon) = \frac{p(\lambda | \Sigma^{\ell}_k;\upsilon)}{p(\lambda|\Sigma^{g}_k;\upsilon)}.\label{eqn::likelrati}
\end{equation}

Finally, the configuration within the candidate set $S^{\mbox{\tiny{TPE}}}_k$ is selected which satisfies:

\begin{equation}
    \lambda \in \arg\max_{\lambda\in S^{\mbox{\tiny{TPE}}}_k} r(\lambda | \Sigma_k;\upsilon) \label{eqn::TPEselect}
\end{equation}
Notice that, since $S^{\mbox{\tiny{TPE}}}_k$ has a finite size, the computational demand to solve the problem in~\eqref{eqn::TPEselect} is typically contained (and controllable by means of $n_{\mbox{\tiny{TPE}}}$). Importantly, \citep{Song2022AGR} shows that this is equivalent (in rank) to maximizing the probability of improvement.
\subsection{\POCAlg Search Phase}\label{sec:search}
At iteration $k$ of \POCAlg, the \textbf{Search} Phase generates $n_{\mbox{\tiny{S}}}$ new configurations $\lambda_{i}$, $i= (k-1)n_{\mbox{\tiny{S}}} + 1, \ldots, (k-1)n_{\mbox{\tiny{S}}} +n_{\mbox{\tiny{S}}}$, and trains the corresponding models $\mathcal{A}\left(\lambda_i\right), \forall i$ using a budget $\delta$, thus using a total budget of $B^{s}_{k} = \delta\cdot n_{s}$. Alg.~\ref{alg:config_search_phase} reports the pseudocode for the search phase.

\begin{algorithm}[htbp]
\caption{Search Phase}
\label{alg:config_search_phase}
\begin{algorithmic}[1]
\State \textbf{Input:} Total budget $B$, 
Remaining budget for the entire optimization $R_k$, iteration $k$, number of configurations $n_{\mbox{\tiny{S}}}$, minimum budget $\delta$, set of previously sampled configurations $\Sigma_{k-1}$, minimum random search probability $\epsilon$, dimensionality of the configuration vector $\boldsymbol{\lambda}$, $d$.
\State \textbf{Output:} Set of sampled configurations $\Sigma_k$
\State $i \gets (k-1)n_{\mbox{\tiny{S}}}+1$. 
\While{$i \le k\cdot n_{\mbox{\tiny{S}}}$ \textbf{and} $R > \delta$}
    \State $q = 
 \max\left(1-0.5\frac{R_k}{B}, 1- \epsilon\right)$.
    \State Sample a value $v$ from a uniform r.v. $V \sim \mathcal{U}(0, 1)$
    \If{$|\Sigma_k| \leq d+1$ or $v > q$}
        \State Sample configuration $\lambda_{k}$ uniformly at random;
    \Else
        \State Sample configuration $\lambda_{k}$ using the TPE sampler;
    \EndIf
    \State Train $\mathcal{A}(\lambda_{k})$ to budget $\delta$, and store $(\lambda_{k}, \mathcal{L}_{\delta}\left(\lambda_k\right)))$
    \State $k \gets k + 1$, $R_k \gets R_{k-1} - \delta$;
    \State Set $i=k$, save candidate $\lambda_i=\lambda_k$ and associated loss $\mathcal{L}^{ik}_{\delta} = \mathcal{L}_{\delta}\left(\lambda_k\right)$, update the set of solutions $\Sigma_k \leftarrow \Sigma_{k-1} \bigcup \{\lambda_{k}\}$.
\EndWhile
\State \Return The set of all sampled configurations $\Sigma_k$
\end{algorithmic}
\end{algorithm}

Until $d+1$ configurations have been sampled by \POCAlg, where $d$ is the dimension of an input $\lambda_i\in\Lambda$, sampling is performed uniformly at random since $d+1$ is the minimum sample required to execute a TPE sampler (section~\ref{sec:backGroundPOCA}). Subsequently, the new configuration $\lambda_{i}$ at iteration $k$ will be sampled using the TPE with probability $q_k = \min\{1-\epsilon, 1-0.5\frac{R_k}{B}\}$, where $R_k$ is the remaining budget at iteration $k$ and $\epsilon\in [0, 0.5]$ is a user-defined parameter that ensures random sampling is retained at a minimum rate. When sampling according to the TPE, \POCAlg uses a multivariate KDE to estimate the relevant sampling probabilities in eqn.~\eqref{eqn::badLikely}-\eqref{eqn::goodLikely}, where each dimension of the KDE corresponds to a specific HPO hyperparameter. For continuous and discrete hyperparameter dimensions, the KDEs use a Gaussian kernel: 
\begin{equation}
\kappa_G(\lambda, \lambda_i ; \upsilon_G) = \frac{1}{\sqrt{2\pi \upsilon_G^2}} \exp{\left[-\frac{1}{2\upsilon_G}(\lambda-\lambda_i)^2\right]},
\end{equation}.

For categorical hyperparameters, we use the Aitchison-Aitken kernel:
\begin{equation}
\kappa_{AA}(\lambda, \lambda_i | \upsilon_{AA}) =
\begin{cases}
    1-\upsilon_{AA} & \lambda = \lambda_i \\
    \frac{\upsilon_{AA}}{N_{\mbox{\tiny{C}}} - 1} & \lambda \neq \lambda_i
\end{cases},
\end{equation}
where $N_{\mbox{\tiny{C}}}$ is the number of options the categorical hyperparameter has.

The kernel function utilized for the TPE likelihood ratio is the product of the kernel functions estimated for each dimension. In other words, if $\{\kappa_j(\lambda, \lambda_i | \upsilon_j\}_{j=1}^d$ are the kernel functions for each hyperparameter (which may be either Gaussian or Aitchison-Aitken), the multivariate kernel function for the likelihood ratio calculation is:
\begin{equation}
\kappa(\lambda, \lambda_i | \upsilon_j \text{ } \forall j) = \prod_{j=1}^d \kappa_j(\lambda, \lambda_i | \upsilon_j),
\end{equation}
where $\upsilon_j = \upsilon_G$ if the kernel $\kappa_j$ is Gaussian and $\upsilon_j = \upsilon_{AA}$ if $\kappa_j$ is Aitchison-Aitken. In our experiments, we use Scott's Rule \citep{scotts_rule} to determine $\upsilon$.

We highlight that alternative kernels that work over hybrid spaces may be considered~\cite{deshwal2021bayesian}. In fact, the \POCAlg implementation allows for such integrations to be easily performed.

If the new configuration $\lambda_{i}$ does not qualify for TPE sampling, it is generated uniformly at random from the set $\Lambda\setminus\Sigma_{k-1}$. Notably, $q_k$ increases linearly from $0.5$ to $1-\epsilon$ as $R_k$ approaches $0$. In other words, the TPE is increasingly utilized as the search progresses, thus trusting more the already identified configurations. By setting $\epsilon > 0$, practitioners can control the level of exploration of the domain $\Lambda$. As $\epsilon$ increases, \POCAlg samples uniformly at random more often, increasing exploration.

Notably, \POCAlg was designed with the flexibility to construct a TPE acquisition utilizing loss sequences that are obtained with different number of epochs across different sampled configurations. This is different from the approach BOHB presented in~\citep{falkner2018bohb}, which builds their TPE only using data from configurations that have been trained to a budget $\beta_{\mbox{\tiny{TPE}}}$ such that for all budgets $\beta > \beta_{\mbox{\tiny{TPE}}}$, less than $d+1$ configurations have been trained to budget $\beta$. Thus, BOHB uses only the subset of the data which has been trained to a large budget in their TPE. While this reduces the number of configurations misclassified to be in $\mathcal{D}^{\ell}$ or $\mathcal{D}^g$, it can lead to the TPE re-exploring unpromising regions of $\Lambda$ since configurations from these regions will not receive a large budget. For this reason, \POCAlg includes data from every configuration in $\Sigma_k$.

\subsection{\POCAlg Evaluation Phase}\label{sec:eval}
At each \POCAlg iteration $k$, the evaluation Phase is responsible for re-sampling a maximum of $n^k_{E} = \mathcal{E}\left(k\right)$ configurations from the set $\Sigma_k$. In fact, only solutions that are promising configurations and are predicted to improve can be sampled and, in the case that such criteria is not met by any of the configurations, the associated budget is ``returned'' to the search phase (section~\ref{alg:config_search_phase}). The subset of solutions that meet both of these criteria is referred to as $\Gamma_k$, and it is initialized as the empty set. Note that, in our implementation of \POCAlg, we use $\mathcal{E}\left(k\right)= k$, but other functions can be explored. If selected, a configuration is assigned an additional $\delta$ budget. Hence, in the case all $n^k_{E} = \mathcal{E}\left(k\right)$ configurations are selected, the phase will deplete a maximum $B_{k}^{E} = \delta \cdot n_E^k$ budget. Alg.~\ref{alg:config_eval_phase} shows the evaluation pseudocode.

\begin{algorithm}
\caption{Evaluation Phase}
\label{alg:config_eval_phase}
\begin{algorithmic}[1]
\State \textbf{Input:} Current iteration $k$, set of previously sampled configurations $\Sigma_{k}$, resampling pool $\Gamma_{k-1}$, budget per increment $\delta$, minimum improvement ratio $\alpha$, best configuration $\lambda^*$
\State \textbf{Output:} Updated set of sampled configurations $\Sigma_k$
\State Compute the set of improving configurations $\Gamma_k$.

\If{$\Gamma_k = \emptyset$} 
    \State Run a search phase with $n_{\mbox{\tiny{S}}} =k$ instead.
\Else
    \State Compute re-selection probabilities for promising configurations $p_{\lambda_{i}}$ where $\lambda_{i} \in C$ 

    \While{$R > \delta$ \textbf{and} $|\Gamma_k| > 0$}
        \State Sample a configuration $\lambda \in C$ according to the assigned probabilities $p_{\lambda_{i}}$
        \State Train $\lambda$ with additional budget $\delta$
        \State Update validation metric of $\lambda$, updating $\Sigma_k$, and recalculate its ARIMA predictions
        \If{$\lambda$ is no longer forecasted to improve}
            \State $\Gamma_k \gets \Gamma_k \setminus \lambda$
        \EndIf
        \State $R \gets R - \delta$
    \EndWhile
\EndIf
\State \Return $\Sigma_k$
\end{algorithmic}
\end{algorithm}

As previously hinted, \POCAlg only re-selects configurations that satisfy the following two criteria: first, the configuration, given its current validation performance, cannot be confidently ruled out as a potential \textit{incumbent} (best sampled configuration); and second, the configuration loss has high likelihood to improve with additional budget. \POCAlg is designed to avoid dedicating resources on configurations that are not promising or that will not provide any additional information when an additional budget is assigned. 

To determine which configurations to potentially re-select, a time-series model is fitted to the sequence of loss values associated to each solution in the set $\Sigma_k$, $\left(\mathcal{L}^{i}_{t}\right)_{t=1}^{T^{i}_k}$.In order to derive the discrete time steps from the training results, we elect a minimum delta budget, $m$ such that, each $m$ epochs \POCAlg collects the validation accuracy associated with the configuration being trained. As an example, when a configuration is first sampled, a number of epochs $\delta$ is allocated for training and a number $\delta/m$ of loss values is collected by \POCAlg, resulting in the loss sequence $\left\lbrace \mathcal{L}^{i}_t \right\rbrace^{\delta/m}_{t=1}$. If a configuration is re-selected for evaluation, the associated loss sequence is augmented with an additional $\delta/m$ loss values.

For the purpose of this paper, we use an Auto-Regressive Integrated Moving Average (ARIMA) model is used as the time-series model; however, building more sophisticated models may be a promising area of improvement to \POCAlg in the future. For configuration $\lambda_{i}$, the ARIMA model provides a prediction of the configuration loss at evaluation index $t+\delta/m$, corresponding to a budget allocated of $\left(t\cdot m + \delta\right)$. In particular, we use the predicted loss $\widehat{\mathcal{L}}^{i}_{t+1}$, and the associated variance $\text{Var}\left[\widehat{\mathcal{L}}^{i}_{t+1}\right]$ to predict the performance gain we may get by allocating more budget to a specific configuration.

The ARIMA($p$, $n_{\mbox{\tiny{d}}}$, $q$) model uses $n_{\mbox{\tiny{d}}}$ levels of differencing to transform a non-stationary time series into a stationary one, and then fits an ARMA($p$, $q$) model on the transformed time-series. Letting $\left(\tilde{\mathcal{L}}_t^i\right)$ be the time-series of the loss sequence $\left(\mathcal{L}_t^i\right)$ after differencing, the ARMA($p$, $q$) model is: 

\begin{equation}
    \label{eq: ARMA model}
    \widetilde{\mathcal{L}}_t^i =  \sum_{l=1}^{p} \theta_{l}\widetilde{\mathcal{L}}_{t-l}^i + \sum_{v=1}^{q} \phi_v \epsilon_{t-v} + \epsilon_t, 
\end{equation}
where $\boldsymbol{\theta}$ are the coefficients for the auto-regression, $\boldsymbol{\phi}$ are the coefficients for the moving average component, and $\boldsymbol{\epsilon}$ are the error terms, which are assumed to be i.i.d. with distribution $\mathcal{N}(0, \sigma_\epsilon)$. Then the prediction $h$ time-steps into the future can be written as~\citep{box-time-series}:
\begin{equation}
\widehat{\widetilde{\mathcal{L}}}^{i}_{t+h} = \sum_{\ell=1}^{p} \theta_{\ell} \widetilde{\mathcal{L}}^i_{t+h-\ell} + \sum_{v=1}^{q} \phi_{v} \epsilon_{t+h-v}.\label{eqn::arimaformpred} 
\end{equation}

An alternate formulation for any ARMA($p$, $q$) time-series model is an MA($\infty$) model. Thus, the model in Equation $\ref{eq: ARMA model}$ can be re-written as:

\begin{equation}
    \label{eq: arma-as-ma-inf}
    \widetilde{\mathcal{L}}_i^t = \sum_{j=0}^\infty \psi_j \epsilon_{t-j},
\end{equation}
where $\psi_0 = 1$. Using this alternate version, the variance of the forecasts can be written as:

\begin{equation}
   \label{eqn::arimaformVar} \mbox{Var}\left(\widehat{\widetilde{\mathcal{L}^i}}_{t+h}\right) = \sigma^2_\epsilon(\sum_{j=0}^h \psi_j).
\end{equation}
Undoing the differencing provides $\widehat{\mathcal{L}}^i_{t+h}$ and $\mbox{Var}\left(\widehat{\mathcal{L}^i}_{t+h}\right)$.

Using these predictions, \POCAlg updates $\Gamma_k$, which is initialized to the empty set, by adding any configuration in $\Sigma_k$ that satisfies the following condition over the predicted performance:

\begin{equation*}
    \widehat{\mathcal{L}}^{i}_{t+\delta/m}\ge \alpha\cdot \mathcal{L}^{i}_{t}, \label{eqn::imrpovCondmean}
\end{equation*}

where $\alpha>1$ is a user-defined hyperparameter.
\POCAlg will continue based on the size of the set $\Gamma_k$.

\paragraph{CASE (a): $\Gamma_k = \emptyset$.} In this case, the configuration evaluation phase is replaced by a configuration search phase, which will sample $k$ (rather than $n_{s}$) configurations and assign them each budget $\delta$.

\paragraph{CASE (b): $|\Gamma_k| > 0$.} In this case, the configurations in $\Gamma_k$ are sampled based on the promise of improvement. \POCAlg 
measures each configuration potential for improving on the incumbent at the $k$-th \POCAlg iteration, $\lambda_k^*$. Specifically, for each configuration $\lambda_{i}\in\Gamma_k$, we derive the random response-improvement as:
\begin{equation*}
    \label{Improvement}
    RI_{i}= \max\{\mathcal{L}^{i}_{t + 1}-\mathcal{L}^*_{k}, 0\}.
\end{equation*}
Note that $\mathcal{L}^{i}_{t + 1}$ is unknown and is a random variable, so $RI_{i}$ is also a random variable. We then consider the expected improvement as:
\begin{equation}
    E\left[RI_i\right]=\int^{\infty}_{\mathcal{L}^*_{k}}l\cdot f_{\mathcal{L}^{i}_{t + \delta/m}}\left(l\right)dl,\label{eqn::expImproARIMA}
\end{equation}

where $\mathcal{L}^{i}_{t + \delta/m}\sim\mathcal{N}\left(\widehat{\mathcal{L}}^{i}_{t + \delta/m},\mbox{Var}\left(\mathcal{L}^{i}_{t + \delta/m}\right)\right)$.
In eqn.~\eqref{eqn::expImproARIMA}, the moments of the normal distribution are calculated using the ARIMA derivation in eqn.~\eqref{eqn::arimaformpred}-\eqref{eqn::arimaformVar}. We define the probability that $\lambda_{i}$ is re-selected as:
\begin{equation}
    P_{i} = \frac{E[RI_{i}]}{\sum_{\lambda_h \in \Gamma_k} E[RI_{h}]},\label{eqn::ImprovProb} i=1, \ldots, |\Gamma_k|.
\end{equation}
Thus, the probability of each configuration being selected is based on total proportion of expected response-improvement belonging to that configuration. The more promising a configuration is, the more likely it is to be selected for further sampling, so long as it is still forecasted to improve. In fact, a fixed configuration $\lambda_i$ will be sampled on expectation $P_i\cdot \mathcal{E}\left(k\right)$, with an expected budget increase for training of $\delta\cdot P_i\cdot \mathcal{E}\left(k\right)$. Similarly, a number of configurations $n_{\mbox{\tiny{E}}}^k \le \mathcal{E}\left(k\right)$ will generally be selected at each evaluation phase. 

\vspace{6pt}
\fbox{\parbox{0.9\linewidth}{\textbf{Example. }Consider the case where the total budget available is $B=800$, and consider Hyperband-based algorithms initialized with a minimum budget of $5$ and a maximum budget of $20$, compared to \POCAlg given $\delta=5$ and $n_{\mbox{\tiny{S}}}=5$. 

\textbf{Hyperband} (and algorithms that repeat identical Hyperbands as their budget allocation framework) would test configurations to a budget of $5\text{, }10 \text{, and }20$ epochs if $\eta=2$. Specifically, the approach would execute $5$ hyperbands, leading to a total of $45$ configurations being sampled. 

Differently, \POCAlg will execute $13$ iterations each with configuration search and evaluation phases. The $k$-th configuration evaluation phase will have a budget $\delta k = 5k$ so that its budget begins smaller than a configuration search phase budget ($25$) and ends over twice as large. Hence, \POCAlg will sample $65$ configurations, exploring the search space significantly more than Hyperband given the same budget.}}

\vspace{6pt}

\POCAlg emphasizes the Search Phase at the beginning of the optimization process to explore the largely uncharted configuration space, as shown in the example. This strategy helps identify promising regions for further exploration at the initial stages of the algorithm. As it advances, POCA shifts focus to the Evaluation Phase to refine and evaluate the most promising configurations. This transition is advantageous because the quality of the sampled configurations improves over time as POCA leverages its TPE. Consequently, allocating significant resources to configurations sampled early on, which are less likely to be optimal, is avoided.

\subsection{Implementation Guidelines}
\label{sec:implement}
In practice, we suggest setting $n_{s}$ and $\delta$ such that \POCAlg dominantly uses configuration search at the start of the optimization process and dominantly uses configuration evaluation towards the end of the process. However, practitioners may decide that search is particularly important for their problem, perhaps if the search space is high-dimensional or otherwise large, and may opt to increase the amount of search towards the end of the process. In such scenarios, the practitioner may desire to raise $\epsilon$ to a larger threshold. However, tuning the hyperparameters of any hyperparameter tuning algorithm, including \POCAlg, is immensely computationally expensive. Hence, further insight into the optimal values of these hyperparameters may be possible.

In practice, most practitioners will parallelize their HPO, and the \POCAlg package provides such capabilities. Specifically, during each configuration search or evaluation phase, the evaluation of the models associated with each selected hyperparameter can be run in parallel. To ensure the reproducibility of our experiments, however, they were all run sequentially.

%% file: sections/5_theory_proof.tex
Hyperparameter optimization (HPO) algorithms face the dual challenge of efficiently allocating computational resources while maintaining exploration guarantees. 

Hyperparameter optimization (HPO) algorithms balance two core challenges: efficiently allocating computational budget to promising configurations while ensuring sufficient exploration of the search space. We contextualize \POCAlg within this framework by establishing its asymptotic equivalence to Hyperband, a state-of-the-art HPO method. To do so, we first formalize the Successive Halving (SH) subroutine central to Hyperband and then demonstrate how \POCAlg replicates its exploration-allocation trade-off.

\paragraph{Successive Halving} Successive Halving (SH) \citep{pmlr-v51-jamieson16} is a multi-stage budget allocation strategy. Given an initial set of configurations, SH iteratively trains them on increasing budgets while discarding underperformers. Formally, for a \emph{halving rate} $\eta \in (0,1)$, each SH bracket operates as follows: \textit{initialize} $n$ configurations with minimum budget $\delta$, \textit{evaluate} all configurations on the current budget, \textit{promote} the top $\eta$ fraction of performers to the next budget tier, increasing their allocation by a factor of $1/\eta$, \textit{repeat} until the total budget $B$ is exhausted. This creates a geometric progression in both budget and elimination. For $\eta=1/3$, for example, each stage triples the budget per configuration while retaining only the top third of candidates. The trade-off lies in choosing $n$ (exploration breadth) versus $\delta$ (initial budget depth): small $\delta$ risks inaccurate early elimination, while small $n$ may overlook optimal regions.

\paragraph{Hyperband} Hyperband \citep{li2018hyperband} generalizes SH by systematically varying $(n, \delta)$ across multiple brackets. Let $\beta_{\max}$ denote the maximum per-configuration budget. Hyperband runs $N_{\text{SH}} = \lfloor \log_{1/\eta} \beta_{\max} \rfloor + 1$ SH brackets, where the $i$-th bracket sets $\delta_i = \beta_{\max} \cdot \eta^{i-1}$ and $n_i = \lfloor B / (\delta_i \cdot \sum_{k=0}^{s_{\max}} \eta^{-k}) \rfloor$ for stages $s_{\max} = \lfloor \log_{1/\eta} (\beta_{\max}/\delta_i) \rfloor$. By cycling through aggressive (high $n$, low $\delta$) and conservative (low $n$, high $\delta$) brackets, Hyperband dynamically balances exploration and exploitation without manual tuning.

\paragraph{\POCAlg} 
We establish the asymptotic equivalence between Hyperband and Progressive Objective Configuration Aggregation (\POCAlg) through structural alignment of their sampling mechanics, budget schedules, and promotion logic. This equivalence holds under mild assumptions about loss convergence and budget allocation to optimal configurations.

\paragraph{Notational Conventions}  
Let $\eta_{\text{SH}} \in (0,1)$ denote Successive Halving's halving rate. Hyperband uses $\eta \triangleq 1/\eta_{\text{SH}} > 1$ as its elimination rate, preserving resource scaling factors while inverting candidate retention ratios. This reciprocal relationship resolves notational ambiguities between the algorithms' original formulations.

\begin{assumption}[Loss Characteristics] \label{assump:loss_charac}
For any configuration $\lambda$ and budget $\beta$:
\begin{enumerate}
\item Square integrability: $\mathbb{E}[\ell_{\lambda,\beta}^2] < \infty$
\item Almost sure convergence: $\mathbb{P}(\lim_{\beta\to\infty} \ell_{\lambda,\beta} = v_\lambda) = 1$
\end{enumerate}
\end{assumption}

These conditions ensure stable comparison of configurations through successive budget amplifications. The almost sure convergence eliminates pathological oscillatory behaviors, while square integrability controls variance accumulation. Practical implementations may relax almost sure convergence to convergence in probability if selection mechanisms exhibit $\eta$-consistency, verifiable via concentration inequalities on empirical loss trajectories.

\begin{assumption}[Budget Growth]\label{asmp:bdgalloc}
The optimal configuration $\lambda^* = \arg\min_\lambda v_\lambda$ receives unbounded resources:
\[
\lim_{B\to\infty} \beta_{\lambda^*}(B) = \infty
\]
where $\beta_{\lambda^*}(B)$ denotes cumulative budget allocated to $\lambda^*$.
\end{assumption}

This prevents permanent starvation of the optimal configuration. Hyperband achieves this through bracket cycling, while \POCAlg uses a persistent $\epsilon$-greedy exploration. 

\begin{proposition}[Sampling Correspondence]\label{prop:sample_eq}
Let $B$ be total budget and $\eta > 1$ Hyperband's elimination rate. For \POCAlg with sampling distribution:
\[
P_{\text{sample}}(\lambda) = \min\left(1 - \frac{1}{\eta}, 1 - \frac{R_k}{2B}\right)P_{\text{TPE}}(\lambda) + \frac{1}{\eta|\Lambda|},
\]
where $R_k = B - \sum_{t=1}^{k-1}\beta_t$, then:
\[
\lim_{B\to\infty} P_{\text{sample}}(\lambda) = \frac{1}{\eta|\Lambda|}
\]
matches Hyperband's uniform bracket initialization.
\end{proposition}

\begin{proof}
Fix iteration $k$ and consider $B \to \infty$. The residual budget $R_k = B - \sum_{t=1}^{k-1}\beta_t$ satisfies:
\[
\frac{R_k}{B} \geq 1 - \frac{(k-1)\beta_{\max}}{B} \to 1,
\]
implying $\min(1 - \frac{1}{\eta}, 1 - \frac{R_k}{2B}) \to \min(1 - \frac{1}{\eta}, \frac{1}{2})$. Setting $\epsilon = 1/\eta$ ensures $1 - \epsilon \geq \frac{1}{2}$ when $\eta \geq 2$, yielding:
\[
\lim_{B\to\infty} P_{\text{sample}}(\lambda) = \frac{1}{2}P_{\text{TPE}}(\lambda) + \frac{1}{\eta|\Lambda|}.
\]
At bracket initialization ($k=1$), TPE defaults to uniform sampling \citep{bergstra2011algorithms} with $P_{\text{TPE}}(\lambda) = 1/|\Lambda|$. Substitution gives:
\[
\lim_{B\to\infty} P_{\text{sample}}(\lambda) = \frac{\eta + 2}{2\eta|\Lambda|}.
\]
Hyperband's configuration count $n_i = \lfloor B/(\beta_{\max}\eta^{-(i-1)}\sum_{k=0}^{s_{\max}}\eta^{-k}) \rfloor$ requires probability $1/n_i$ per configuration. Equating terms reveals:
\[
\frac{\eta + 2}{2\eta|\Lambda|} = \frac{1}{n_i} \implies |\Lambda| = \frac{n_i(\eta + 2)}{2\eta}.
\]
Substituting Hyperband's geometric series $\sum_{k=0}^{s_{\max}}\eta^{-k} = \frac{1 - \eta^{-(s_{\max}+1)}}{1 - \eta^{-1}}$ under termination rule $s_{\max} = \lfloor \log_\eta(\beta_{\max}/\delta_i) \rfloor$ confirms equality when $\delta_i = \beta_{\max}\eta^{-(i-1)}$. 
\end{proof}

The sampling correspondence relies on TPE's uniformity at initialization, justified by absent prior data. Hyperband's bracket cycling naturally enforces this condition through resource recycling. This alignment ensures both algorithms explore configuration spaces with equivalent breadth in the asymptotic regime.

An important requirement for asymptotic convergence is that all configurations are evaluated infinitely often. The next proposition provides a general condition for this property.

\begin{proposition}[Infinite Resampling]\label{Re-Sample Proposition}
Let $\mathcal{A}$ be an arbitrary HPO algorithm, and let $m \in \mathbb{Z^+}$ be arbitrary. Let $S = \{\lambda | \lambda \text{ was drawn by } \mathcal{A} \text{ during configuration selection}\}$, and let $\lambda \in S$ be arbitrary. Let $A_i$ be the event that $\mathcal{A}$ samples $\lambda$ for further evaluation in its $ith$ selection in the future by some selection process, where each selection is independent of all others. Then, if $\sum_{i=1}^\infty P[A_i] = \infty$, then $\mathcal{A}$ will sample $\lambda$ at least $m$ times given sufficient budget.
\end{proposition}

\proof{Proof.}
    This follows directly from the Borel-Cantelli Zero-One Law since the sampling procedure is independent. In fact, we can say that $\lambda$ will be sampled infinitely often. 
\endproof

We now show \POCAlg satisfies this condition.

\begin{lemma}[POCA Satisfies Resampling]\label{lemma:poca_resample}
Let $\Sigma_k$ denote the set of configurations sampled up until iteration $k$. Every configuration $\lambda \in \Sigma_k$ satisfies $\sum_{k=1}^\infty P(\text{select } \lambda \text{ in } k) = \infty$.
\end{lemma}

\proof{Proof.}
Consider the definition of POCA where each configuration $\lambda \in S$ is assigned a probability of re-selection of $\frac{1}{|S|\log|S|} + p_{\lambda}$, where the first term is a uniform allocation of a small pool of probability and the second term $p_\lambda$ allocates additional probability according to the merit of the configuration.

Fix an arbitrary $\lambda \in S$, and let $U_i$ be the event that $\lambda$ was selected by the uniform allocation process as the $i^{th}$ configuration selected. (This can be seen as the product of two events: whether to select uniformly or by merit, and when selecting uniformly, which configuration to select.) $P(U_i) \ge \frac{1}{i n_{CS} \log i n_{CS}}$ by considering only the first selection of each configuration evaluation stage. Hence, $\sum_{i=1}^\infty P(U_i) \ge \sum_{i=1}^\infty \frac{1}{i n_{CS} \log i n_{CS}} = \infty$. 
\endproof

Lemma \ref{lemma:poca_resample} guarantees that even suboptimal configurations are resampled infinitely often, ensuring \POCAlg does not permanently discard any candidate.

\begin{lemma}[Budget Schedule Congruence]\label{lemma:budget_align}
After $k$ promotions, both algorithms allocate:
\[
\beta_{\lambda}(k) = \beta_{\min}\eta^{-k}.
\]
\end{lemma}

\proof{Proof.}
By induction. The base case $\beta_{\lambda}(0) = \beta_{\min}$ holds by initialization. Assume $\beta_{\lambda}(k) = \beta_{\min}\eta^{-k}$. For Hyperband survivors:
\[
\beta_{\lambda}^{\text{HB}}(k+1) = \eta\beta_{\lambda}^{\text{HB}}(k) = \beta_{\min}\eta^{-(k+1)}.
\]
For \POCAlg, promotion threshold $\alpha = \eta$ and almost sure convergence (Assumption \ref{assump:loss_charac}) yield:
\[
\frac{ \widehat{\mathcal{L}}_{\lambda,\beta(k)}}{\widehat{\mathcal{L}}_{\lambda^*,\beta(k)}} \geq \eta \implies \beta_{\lambda}^{\text{\POCAlg}}(k+1) = \eta\beta_{\lambda}^{\text{\POCAlg}}(k).
\] 
\endproof

\begin{theorem}[Constructive Equivalence]\label{thm:constructive}
Under Assumptions \ref{assump:loss_charac}-\ref{asmp:bdgalloc}, Hyperband and \POCAlg exhibit identical asymptotic behavior when parameterized with \POCAlg exploration rate $\epsilon = 1/\eta$, promotion threshold $\alpha = \eta$, and phase count $N = \lfloor \log_\eta \beta_{\max} \rfloor + 1$.
\end{theorem}

\proof{Proof.}
The operational equivalence emerges through four mechanisms. First, Hyperband's bracket structure bijectively maps to \POCAlg's phased execution through budget alignment $B_i = (\lfloor \log_\eta \beta_{\max} \rfloor + 1)\beta_{\max}$, creating similar iteration cycles. Second, Proposition \ref{prop:sample_eq} establishes sampling consistency by proving $\lim_{B\to\infty} P_{\text{sample}}(\lambda) = 1/(\eta|\Lambda|)$, which mirrors Hyperband's uniform initialization distribution. 

Third, Lemma \ref{lemma:budget_align} guarantees synchronized budget schedules $\beta_{\lambda}(k) = \beta_{\min}\eta^{-k}$ across both algorithms, ensuring matched resource amplification trajectories after $k$ promotions. Fourth, Assumption \ref{assump:loss_charac} enables equivalent survivor selection: \POCAlg's threshold $\alpha = \eta$ replicates Hyperband's top-$\eta^{-1}$ retention through almost sure loss convergence. 

Persistent exploration is maintained through Proposition \ref{lemma:poca_resample}'s infinite resampling guarantee, while finite deviations vanish asymptotically by the Law of Large Numbers. These components collectively induce identical configuration space traversal patterns as $B \to \infty$. 
\endproof

%% file: sections/6_experiments.tex
We ran two classes of experiments with the objective to compare \POCAlg to state-of-the-art HPO algorithms. The main objective is to understand the impact of the evaluation phase as compared to the fixed hyperband mechanism and the impact of the configuration generator. 

\noindent\textbf{Benchmark Algorithms. }In particular, we chose the following algorithms as benchmarks for our \POCAlg:
\begin{itemize}
    \item DEHB: a Hyperband-based algorithm that uses Differential Evolution as a mechanism for configuration sampling~\citep{awad2021dehb};
    \item SMAC: a Hyperband-based algorithm that uses Random Forests to predict and search for configurations~\citep{lindauer2022smac3};
    \item BOHB: a Hyperband-based algorithm that, like our \POCAlg, uses a TPE-based mechanism to generate new configurations~\citep{falkner2018bohb}.
\end{itemize}
For all three, the minimum budget was set to $\delta_{\mbox{\tiny{min}}}=5$ epochs, and the maximum budget was set to $\delta_{\mbox{\tiny{max}}}=45$ epochs, with ``halving'' ratio $\eta=\frac{1}{3}$. The remainder of the algorithms hyperparameters were set to their default settings. 

\noindent\textbf{Data sets. }The first set of experiments, presented in Section~\ref{sec:yahpo}, consists of the $34$ YAHPO Gym data sets for the training of a neural network hyperparmeters. This is a commonly used benchmark, and its $34$ different datasets ensures that the algorithms are tested under a variety of scenarios. Notably, while the tabular nature of these experiments does not represent real-life scenarios, the pre-trained nature of the data set allows to macro-replicate the study many times, which is not otherwise feasible without using tabular or surrogate datasets.

As second class of experiments, we present the MNIST data set in Section~\ref{sec:mnist}. This provides an experiment on a popular dataset where the search space is not reduced to a discrete set, as in the YAHPO experiments, which is more realistic to real-world HPO.

\noindent\textbf{Performance Metrics. } The algorithms are evaluated based on the test accuracy on their incumbent at each budget level. For the YAHPO experiments, the test accuracy are readily available. For the MNIST experiment, incumbent configurations were re-trained and then evaluated on the test dataset.

\subsection{YAHPO Gym Experiments}
\label{sec:yahpo}
The first set of experiments uses a tabular benchmark referred to as the YAHPO Gym Benchmark \citep{yahpo_overall} on the LCBench scenario \citep{lcbench}. For scenario contains 34 different datasets for which a common multi-layer perceptron is used to learn the task appropriate for the dataset. For each dataset, the hyperparameters needing tuning are the same and are shown in Table~\ref{tab:yahpo_hps}. The datasets provide pre-trained results for a finite search space, featuring $2000$ configurations. In particular, loss values are available for models trained under a budget $\beta^i\in\left\lbrace 1,\ldots, 52\right\rbrace$ (i.e., no model can be evaluated for more than $52$ epochs). We set a total budget $B=1000$, and perform $90$ macro-replications for \POCAlg and the competitors. The total budget represents roughly a hundredth of the budget necessary to fully search the entire space.

\begin{table}
\centering
\caption{Hyperparameter Configuration Search Space for YAHPO Gym Experiments}
\label{tab:yahpo_hps}
\begin{tabular}{llll}
\toprule
Hyperparameter & Range & Distribution & Type\\
\midrule
Number of layers & [1, 5] & Uniform & Int \\
Max. number of units & [64, 512] & Log Uniform & Int \\
Batch Size & [16, 512] & Log Uniform & Int \\
Learning Rate & [1e-4, 1e-1] & Log Uniform & Float \\
L2 Regularization & [1e-5, 1e-1] & Uniform & Float \\
SGD Momentum & [0.1, 0.99] & Uniform & Float \\
Max. Dropout Rate & [0, 1] & Uniform & Float \\
\bottomrule
\end{tabular}
\end{table}

\POCAlg, SMAC, BOHB, and DEHB were also compared in this experiment. For \POCAlg, we set the budget increment $\delta = 5$, the minimum random search sampling rate $\epsilon=0.05$, and the minimum improvement ratio $\alpha=1.05$. Further, for the ARIMA models on a configuration's loss values, we used an ARIMA$(3, 1, 0)$ model. 

\begin{figure}
\centering
\includegraphics[width=0.65\textwidth]
{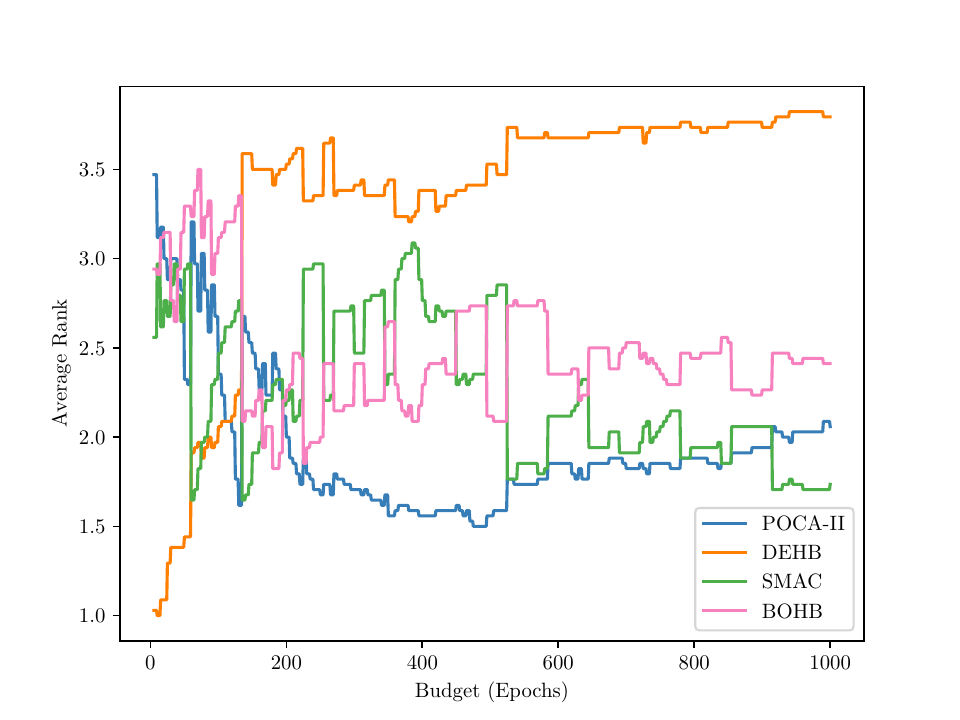}
\caption{Average ranks of \POCAlg, SMAC, DEHB, and BOHB on the YAHPO Gym LCNet Scenario}
\label{fig:yahpo-ranks}
\end{figure}

To aptly and succinctly compare the algorithms across the $34$ datasets, Figure~\ref{fig:yahpo-ranks} shows the average ranks of the algorithms for varying number of training epochs. Specifically, for each dataset, the test accuracy of the best-found configuration is found for every budget level and is averaged across the 90 macro-replications. Then, at each budget level, these accuracies are ranked from best (rank $1$) to worst (rank $4$). Then, at each budget level, the ranks are averaged over the the $34$ experiments. Figure~\ref{fig:yahpo-ranks} shows how, once the algorithms can be reasonably differentiated, beginning at roughly $200$ epochs, \POCAlg has the best average rank until around $900$ epochs. After that point, SMAC's average performance is slightly better than \POCAlg. Considering most practitioners assign a relatively small budget for HPO in practical scenarios, \POCAlg would often be the preferred algorithm. However, for practitioners that are willing to spend significant resources and several days of computing time towards tuning their hyperparameters, SMAC may be the best performing algorithm. In fact, switching from \POCAlg to SMAC may lead to even better configurations (i.e., implementing hybrid HPO algorithms).


Figure~\ref{fig:yahpo-accs} shows the average test accuracy of each of the four algorithms by epoch for $2$ of the $34$ datasets in the experiment, showing one experiment where \POCAlg performs well and another where it is sub-optimal. While there are some datasets where the algorithms perform similarly throughout the process, there is often a meaningful difference in test accuracies, such as in Dataset 189354, where the difference in \POCAlg's incumbent's test accuracy is higher than that of its competitors by over two percentage points at 500 epochs.

\begin{figure}[h]
 \begin{subfigure}{0.49\textwidth}
     \includegraphics[width=\textwidth]
     {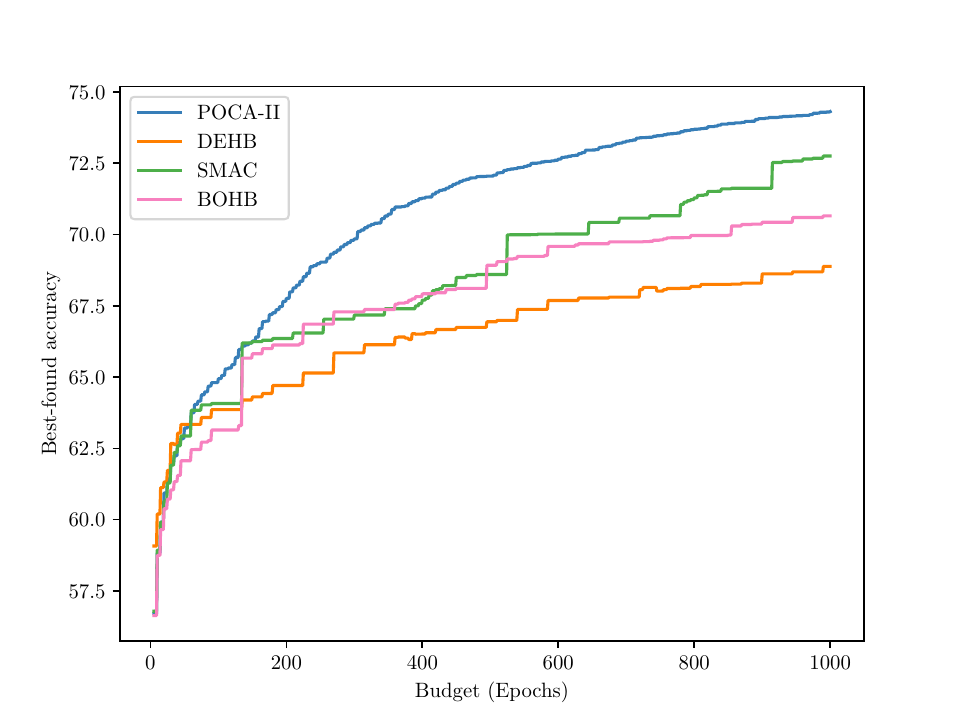}
     \caption{Dataset 189354.}
 \end{subfigure}
 \medskip
 \begin{subfigure}{0.49\textwidth}
     \includegraphics[width=\textwidth]
     {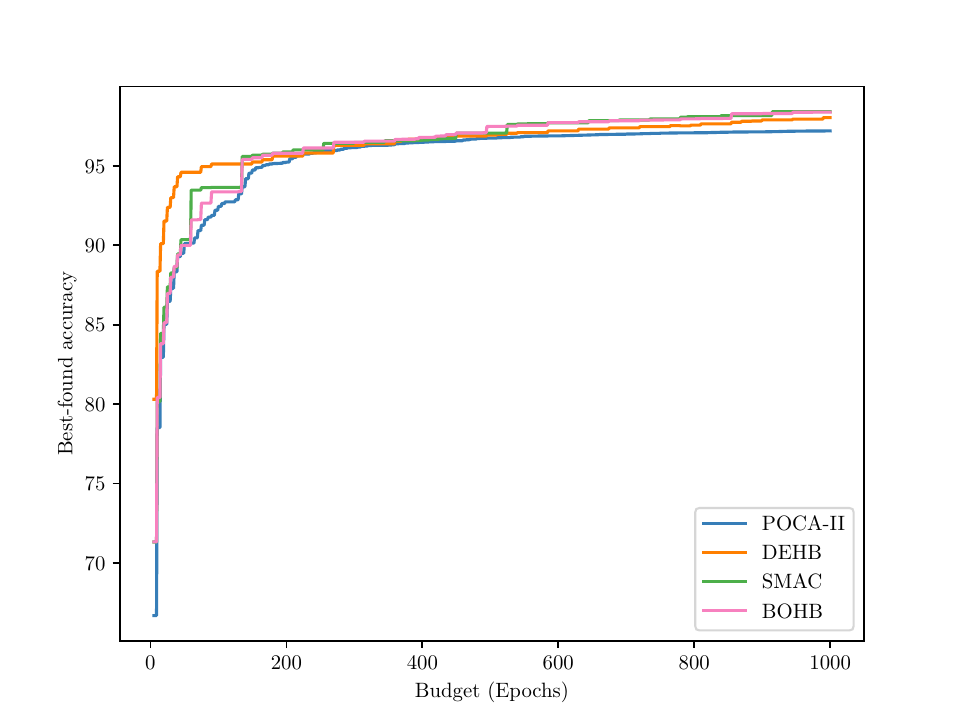}
     \caption{Dataset 34539.}
 \end{subfigure}
 \caption{Average Test Accuracy over $90$ Replications against Epochs for representative Datasets}
 \label{fig:yahpo-accs}
\end{figure}

In addition, Tables \ref{tab:DS1-10}, \ref{tab:DS11-20}, \ref{tab:DSlast} report the average test accuracy and standard deviation of the test accuracy at 300 epochs for each of the 34 datasets. Tests of statistical significance were performed on the algorithm's average test accuracies. First, a Bartlett's test was used to determined if the variances were equal for that dataset. Then, if they were, an ANOVA test was conducted to determine if the difference between algorithms were statistically significant, and if they were, a Tukey Pairwise Test determined which algorithms performed statistically significantly better than the other algorithms. If the variances were not equal, then a Welch's ANOVA test, followed by a Games-Howell Pairwise test were employed. \POCAlg was statistically superior over all three other algorithms in 11 datasets. In an additional 11 datasets, \POCAlg was statistically signifcantly better than at least one other algorithm, and in 2 datasets, \POCAlg was statistically significantly inferior to at least 1 other algorithm. No other algorithm was statistically significantly superior to all other algorithms in any of the 34 datasets. All tests were performed with level $0.05$. Among the ANOVA and Welch's ANOVA tests, 25 of the 34 tests produced a significant result, far higher than the expected number of false positives under the null hypothesis. It is notable that the standard deviation around \POCAlg's test accuracy is smaller than that of the competing algorithms, showing that \POCAlg's results are more consistent.

\begin{table}[htbp]
  \centering
  \caption{Average test accuracy and standard error for Datasets 1-12 of YAHPO Gym's LCNet Scenario. Bolded values indicate that the algorithm performed statistically significantly better than the un-bolded algorithms. $^*$ indicates \POCAlg was statistically significantly better than at least one of the other algorithms. $^\dagger$ indicates the \POCAlg was statistically significantly worse than at least one other algorithm.}
  \label{tab:DS1-10}
  \resizebox{\textwidth}{!}{
    \begin{tabular}{l|l|rrrrrrrrrrrr}
    \toprule
          \up\down Method&       & \multicolumn{1}{l}{DS-$1^*$} & \multicolumn{1}{l}{DS-$2^*$} & \multicolumn{1}{l}{DS-3} & \multicolumn{1}{l}{DS-$4^*$} & \multicolumn{1}{l}{DS-$5^*$} & \multicolumn{1}{l}{DS-$6^*$} & \multicolumn{1}{l}{DS-$7^\dagger$} & \multicolumn{1}{l}{DS-8} & \multicolumn{1}{l}{DS-$9^\dagger$} & \multicolumn{1}{l}{DS-$10^\dagger$} & \multicolumn{1}{l}{DS-$11^*$} & \multicolumn{1}{l}{DS-12}\\
    \midrule
    \up\down\multirow{2}[0]{*}{POCA II} & $\mu$ &  99.01 & 69.79 & 96.15 & \textbf{86.00} & 95.81 & 89.70 &     95.12 & 90.77 & 95.77 & 95.99 & 77.36 & 74.41\\
    \cline{2-14}
           \up\down& se    & 0.03 & 0.33 & 0.10 & 0.19 & 0.04 & 0.17 & 0.14 & 0.10 & 0.07 & 0.04 & 0.26 & 0.41 \\
    \midrule
    \up\down\multirow{2}[0]{*}{SMAC} & $\mu$ & 98.52 & 68.64 & 96.48 & \textbf{85.32} & 95.56 & 88.97 & \textbf{96.75}  & 91.07 & 96.25 & 96.20 & 76.63 & 73.77 \\
    \cline{2-14}
          \up\down& se    & 0.10 & 0.53 & 0.19 & 0.31 & 0.09 & 0.35 & 0.17 & 0.13 & 0.12 & 0.08 & 0.42 & 0.53\\
    \midrule
    \up\down\multirow{2}[0]{*}{BOHB} & $\mu$ & 98.49 & 67.73 & 96.49 & \textbf{86.36} & 95.78 & 89.12 & \textbf{96.81} & 91.00 & 96.17 & 96.30 & 76.47 & 73.38\\
    \cline{2-14}
          \up\down& se    & 0.13 & 0.51 & 0.21 & 0.28 & 0.09 & 0.35 & 0.13 & 0.15 & 0.16 & 0.08 & 0.50 & 0.55\\
    \midrule
    \up\down\multirow{2}[0]{*}{DEHB} & $\mu$ & 98.20 & 66.53 & 96.29 & 83.66 & 95.56 & 87.71 & \textbf{96.43} & 90.82 & 96.47 & 96.24 & 75.56 & 74.65\\
    \cline{2-14}
          \up\down& se    & 0.15 & 0.56 & 0.26 & 0.39 & 0.11 & 0.39 & 0.17 & 0.16 & 0.13 & 0.09 & 0.49 & 0.59 \\
    \bottomrule
    \end{tabular}}
  
\end{table}%

\begin{table}[htbp]
  \centering
  \caption{Average test accuracy and standard error for Datasets 13-23 of YAHPO Gym's LCNet Scenario. Bolded values indicate that the algorithm performed statistically significantly better than the un-bolded algorithms. $^*$ indicates \POCAlg was statistically significantly better than at least one of the other algorithms. $^\dagger$ indicates the \POCAlg was statistically significantly worse than at least one other algorithm.}
  \label{tab:DS11-20}
  \resizebox{\textwidth}{!}{
    \begin{tabular}{l|lrrrrrrrrrrrr}
    \toprule
          \up\down &       & \multicolumn{1}{l}{DS-13} & \multicolumn{1}{l}{DS-$14^*$} & \multicolumn{1}{l}{DS-15} & \multicolumn{1}{l}{DS-16} & \multicolumn{1}{l}{DS-$17^*$} & \multicolumn{1}{l}{DS-18} & \multicolumn{1}{l}{DS-$19^*$} & \multicolumn{1}{l}{DS-$20^*$} & \multicolumn{1}{l}{DS-$21^*$} & \multicolumn{1}{l}{DS-22} & \multicolumn{1}{l}{DS-$23^*$}  \\
    \hline
    \up\down \multirow{2}[0]{*}{POCA II} & $\mu$ & 84.05 &  \textbf{75.69} & 95.23 & 81.64 & \textbf{69.25} & 68.43 & \textbf{27.11} & 62.42 & 54.43 & 90.66 & 99.33\\
    \cline{2-13}
          \up\down & se & 0.34 & 0.29 & 0.04 & 0.22 & 0.29 & 0.52 & 0.32 & 0.34 & 0.43 & 0.14 & 0.02 \\
    \hline
    \up\down \multirow{2}[0]{*}{SMAC} & $\mu$ & 84.35 &     72.12 & 94.86 & 80.70 & 67.49 & 68.84 &       22.15 & 61.25 & 53.04 & 90.43 & 99.21 \\
    \cline{2-13}
          \up\down & se    & 0.57 & 0.68 & 0.25 & 0.36 & 0.35 & 0.57 & 0.52 & 0.41 & 0.51 & 0.25 & 0.04 \\
    \hline
    \up\down \multirow{2}[0]{*}{BOHB} & $\mu$ & 83.95 &     73.24 & 96.97 & 81.01 & 67.55 & 69.01 & 22.32 & 60.44 & 52.96 & 90.20 & 99.24\\
    \cline{2-13}
          \up\down & se    & 0.60 & 0.67 & 0.24 & 0.40 & 0.28 & 0.62 & 0.40 & 0.40 & 0.49 & 0.23 & 0.04 \\
    \hline
    \up\down \multirow{2}[0]{*}{DEHB} & $\mu$ & 82.42 &     72.79 & 94.23 & 81.08 & 66.96 & 69.10 & 21.29 & 60.27 & 51.42 & 90.23 & 99.01 \\
    \cline{2-13}
          \up\down & se    & 0.69 & 0.78 & 0.41 & 0.48 & 0.33 & 0.65 & 0.52 & 0.37 & 0.42 & 0.29 &  0.06\\
    \bottomrule
    \end{tabular}}
 
\end{table}%

\begin{table}[htbp]
  \centering
  \caption{Average test accuracy and standard error for Datasets 24-34 of YAHPO Gym's LCNet Scenario. Bolded values indicate that the algorithm performed statistically significantly better than the un-bolded algorithms. $^*$ indicates \POCAlg was statistically significantly better than at least one of the other algorithms. $^\dagger$ indicates the \POCAlg was statistically significantly worse than at least one other algorithm. }
  \label{tab:DSlast}
  \resizebox{\textwidth}{!}{
  
  \begin{tabular}{l|lrrrrrrrrrrr}
    \toprule
          \up\down &       & \multicolumn{1}{l}{DS-$24^*$} & \multicolumn{1}{l}{DS-$25^*$} & \multicolumn{1}{l}{DS-$26^*$} & \multicolumn{1}{l}{DS-$27^*$} & \multicolumn{1}{l}{DS-28} & \multicolumn{1}{l}{DS-29} & \multicolumn{1}{l}{DS-$30^*$} & \multicolumn{1}{l}{DS-$31^*$} & \multicolumn{1}{l}{DS-$32^*$} & \multicolumn{1}{l}{DS-$33^*$} & \multicolumn{1}{l}{DS-$34^*$}\\
    \hline
    \up\down \multirow{2}[0]{*}{POCA II} & $\mu$ & 78.34 & \textbf{67.10} & \textbf{69.85} & \textbf{83.18} & 93.21 & 66.72 & \textbf{73.58} & \textbf{89.00} & \textbf{85.23} & \textbf{87.97} & \textbf{75.03}\\
    \cline{2-13}
          \up\down & se    & 0.21 & 0.17 & 0.29 & 0.09 & 0.08 & 0.20 & 0.88 & 0.20 & 0.23 & 0.12 & 0.27\\
    \hline
    \up\down \multirow{2}[0]{*}{SMAC} & $\mu$ & 77.66 &     66.20 & 67.18 & 82.32 & 93.52 & 66.82 &       66.56 & 87.43 & 83.56 & 86.14 & 72.21\\
    \cline{2-13}
          \up\down & se & 0.31 & 0.25 & 0.33 & 0.20 & 0.13 & 0.31 & 1.29 & 0.43 & 0.36 & 0.17 & 0.41\\
    \hline
    \up\down \multirow{2}[0]{*}{BOHB} & $\mu$ & 77.46 & 66.31 & 67.29 & 82.47 & 93.51 & 66.28 & 66.24 & 87.67 & 83.72 & 86.19 & 71.92\\
    \cline{2-13}
          \up\down & se & 0.32 & 0.22 & 0.35 & 0.22 & 0.16 & 0.26 & 1.12 & 0.44 & 0.35 & 0.16 & 0.41\\
    \hline
    \up\down \multirow{2}[0]{*}{DEHB} & $\mu$ & 76.50 &  65.75 & 65.85 & 81.89 & 93.40 & 66.33 & 59.76  & 86.98 & 83.21 & 85.74 & 70.70\\
    \cline{2-13}
          \up\down & se    & 0.44 & 0.35 & 0.48 & 0.29 & 0.17 & 0.28 & 2.11 & 0.42 & 0.48 & 0.22 & 0.41\\
    \bottomrule
    \end{tabular}}
 
\end{table}%

\subsection{MNIST Experiment}\label{sec:mnist}
While tabular datasets are convenient for testing several HPO algorithms allowing for extensive macro-replicated studies, the experiments that most closely resemble real-world scenarios are those where the training is performed as the algorithm progresses, and the search space is not discretized. 


The model being trained is a CNN with three convolutional layers and three dense layers with a dropout layer after each convolutional layer. The hyperparameters and their search spaces are shown in Table~\ref{tab:mnist-hps}. Each HPO algorithm was given a total budget of 2500 epochs, and this process was replicated fifteen times each. Due to the extensive cost associated with replicating this degree of neural network training, only SMAC and \POCAlg were compared in this experiment. 

To evaluate the algorithms using the test dataset, the best-found configuration of each algorithm was re-trained to 35 epochs and then the test accuracy was computed. This evaluation was performed for the best-found configuration of each algorithm after $100$, $250$, $500$, $750$, $1000$, $1500$, $2000$, and $2500$ epochs.
\begin{table}[H]
\centering

\caption{Hyperparameter Configuration Search Space for MNIST Experiments.}
\label{tab:mnist-hps}
\begin{tabular}{lll}
\toprule
\up\down Hyperparameter & Range  & Type\\
\midrule
\up Optimizer & Adam, SGD & Categorical \\
Learning Rate & [1e-5, 0.2] & Float \\
SGD Momentum & [0, 1] & Float \\
Weight Decay & [0, 0.5] & Float \\
Batch Size & [8, 1028] & Int \\
\down Max Dropout & [0, 1] & Float \\
\bottomrule
\end{tabular}
\end{table}

Figure~\ref{fig:mnist-test-acc} shows boxplots of the test accuracies of \POCAlg and SMAC at various benchmarks, along with the sample variance from the $15$ macro-replications. These numerical results are also reported in Table~\ref{tab:MNIST}. \POCAlg and SMAC's mean and median test accuracy are similar at every budget level, but \POCAlg, due to its philosophy of focusing on configuration search towards the start of the HPO process, has a much smaller variance in performance at low-budgets than SMAC does. SMAC, however, has a smaller variance at large budgets. These results are consistent with the observations for the data set in section~\ref{sec:yahpo}, which showed that \POCAlg outperforms SMAC during most reasonable budgets for HPO, but SMAC performs well once extremely expensive training can be performed.

\begin{table}[htbp]
  \centering

  \caption{Average test accuracy and standard error for MNIST Experiments at various epochs}
  \label{tab:MNIST}
    \begin{tabular}{l|lrrrrrrrr}
    \toprule
          \up\down & & \multicolumn{1}{l}{100} & \multicolumn{1}{l}{250} & \multicolumn{1}{l}{500} & \multicolumn{1}{l}{750} & \multicolumn{1}{l}{1000} & \multicolumn{1}{l}{1500} & \multicolumn{1}{l}{2000} & \multicolumn{1}{l}{2500} \\
    \hline
    \up\down \multirow{2}[0]{*}{POCA II} & $\mu$ & 0.9871 & 0.9878 & 0.9886 & 0.9887 & 0.9885 & 0.9881 & 0.9885 & 0.9886 \\
    \cline{2-10}
          \up\down & se  & 1.1e-3 & 8e-4 & 6e-4 & 6e-4 & 7e-4 & 5e-4 & 6e-4 & 6e-4 \\
    \hline
    \up\down \multirow{2}[0]{*}{SMAC} & $\mu$ &  
    0.9804 & 0.9871 & 0.9886 & 0.9890 & 0.9887 & 0.9891 & 0.9904 & 0.9897\\
    \cline{2-10}
          \up\down & se  & 3.8e-3 & 1.3e-3 & 8e-4 & 6e-4 & 9e-4 & 5e-4 & 4e-4 & 3e-4 \\
    \bottomrule
    \end{tabular}
\end{table}

\begin{figure}
\centering
\includegraphics[width=0.5\textwidth]
{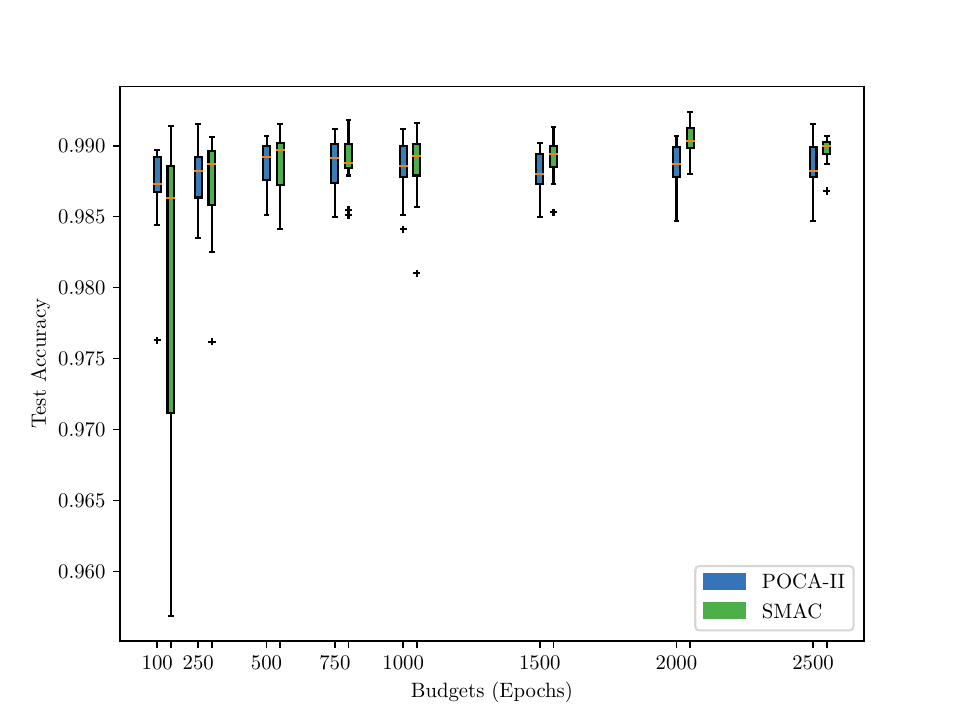}
\caption{Boxplot of Test Accuracies of the Best-Found Configurations of the \POCAlg and SMAC algorithms after various budgets of HPO.}
\label{fig:mnist-test-acc}
\end{figure}

%% file: sections/7_conclusion_futureres.tex
In this paper we propose for the first time the HPO algorithm \POCAlg. \POCAlg differs from the Hyperband and Successive Halving literature by explicitly separating the search and evaluation phases and utilizing exploration and exploitation principles to both. Such distinction results in a highly  flexible scheme for managing hyper parameter optimization budget by focusing on search (i.e., generating competing configurations) towards the start of the HPO process while increasing the evaluation effort as the HPO comes to an end.

\POCAlg was compared to state of the art approaches SMAC, BOHB and DEHB. Our algorithm shows superior performance in low-budget hyperparameter optimization regimes. Since many practitioners do not have exhaustive resources to assign to HPO, it has wide applications to real-world problems. Moreover, the empirical evidence showed how \POCAlg demonstrates higher robustness and lower variance in the results. This is again very important when considering realistic scenarios with extremely expensive models to train.

In fact, the improved long-run performance of SMAC, and the empirical analysis of \POCAlg  leads to three directions for future research: (i) adaptive switching schemes from explore-exploit to hyperband; (ii)  parallel versions of the algorithm; (iii) new kernels and acquisition functions for the search and evaluation phases. Concerning the first direction, we could combine the strong performance in low-budget HPO with equivalently strong performance when considerable resources are available for HPO. In this setting, algorithms that can adaptively switch between budget allocation and configuration generation schemes may be attractive. Further research avenues include the \POCAlg parallelization, and distribution. As an example, distrbuted and multi-agent implementations can help in running instances of \POCAlg in different subspaces and iteratively focus on more promising subregions in the hyperparameter space. Concerning surrogates for the configuration response, the ARIMA time-series models could be replaced with different filters and the acquisition functions for the evaluation may be extended to consider the response improvement for the single configuration ignoring the remainder of the configurations in the promising set. This would allow to further assign budget to ``newer'' solutions. Moreover, different kernels can be embedded within \POCAlg while still being able to embed the two-stage resource allocation scheme with search and evaluation. 

Finally, we note that the modular implementation of \POCAlg code base will provide HPO researchers with a robust implementation framework and quick development setup for experimenting the desired variations. 